\documentclass{article}

\usepackage[a4paper]{geometry}

\usepackage{graphicx} 

\usepackage[T1]{fontenc}

\usepackage{algorithm}
\usepackage{algorithmic}

\usepackage{breakcites}
\usepackage{authblk}
\usepackage{caption}
\usepackage{subfigure}
\usepackage{epsf}
\usepackage{amsmath}
\usepackage{amssymb}
\usepackage{amsfonts}
\usepackage{amsthm}
\usepackage{dsfont}
\usepackage{enumitem}
\usepackage{float}
\usepackage{cases}
\usepackage{verbatim}
\usepackage{hyperref}
\usepackage{tikz}
\usepackage{wrapfig}
\usetikzlibrary{matrix}
\usepackage{url} 
\usepackage{color}
\usepackage{xfrac}
\usepackage{float}
\usepackage{enumitem}
\usepackage{stackrel}

\newcommand{\A}{\mathcal A}

\newcommand{\E}{\mathbb E}
\newcommand{\X}{\mathcal X}
\newcommand{\Y}{\mathcal Y}

\newcommand{\N}{\mathcal N}

\newcommand{\htheta}{\hat\theta}

\DeclareMathOperator*{\argmax}{arg\,max}

\newtheorem{theorem}{Theorem}

\title{Regression with n$\to$1 by Expert Knowledge Elicitation}

\author{Marta Soare}
\author{Muhammad Ammad-ud-din}  
\author{Samuel Kaski}
\affil{Helsinki Institute for Information Technology HIIT,\\
Department of Computer Science, Aalto University \\
\texttt{first.last@aalto.fi}}

\date{}

\begin{document}

\maketitle

\begin{abstract}
  
We consider regression under the ``extremely small $n$ large $p$'' condition. In particular, we focus on problems with so small sample sizes $n$ compared to the dimensionality $p$, even $n\to 1$, that predictors cannot be estimated without prior knowledge. Furthermore, we assume all prior knowledge that can be automatically extracted from databases has already been taken into account. This setup occurs in personalized medicine, for instance, when predicting treatment outcomes for an individual patient based on noisy high-dimensional genomics data. A remaining source of information is expert knowledge which has received relatively little attention in recent years. We formulate the inference problem of asking expert feedback on features on a budget, present experimental results for two setups: ``small $n$'' and ``n=1 with similar data available'', and derive conditions under which the elicitation strategy is optimal. Experiments on simulated experts, both on simulated and genomics data, demonstrate that the proposed strategy can drastically improve prediction accuracy.
 
\end{abstract}

\medskip

\section{Introduction}

In the ``small $n$ large $p$'' regression setting, the task is to make predictions based on few noisy samples of high-dimensional data. It is impossible to address this problem without relying on prior knowledge. Typically, prior knowledge is represented by known structures in data, such as groupings of variables in pathways, or sparsity. Still, for the extreme case of $n \to 1$, none of them is sufficient, and not even their combination. In this case, a potential source of additional useful knowledge comes from human expertise, which is usually expensive to extract. In this paper we address the problem of how to efficiently elicit expert knowledge, under a restricted feedback budget, making the simplifying assumption that the user is able to provide exact information to queries. 

\subsection{Dealing with small n large p}

``Small $n$ large $p$'' data (also known as ``fat data'') is characterised by a large number of predictors $(p)$ that need to be estimated from few data (small sample size $n$). This situation is typical in medical data, where observations (such as drug responses) are very scarce, a very large number of potentially relevant covariates is available, from genomics measurements for instance, and additional data can only be obtained at a high cost. 

To tackle this problem, and to efficiently constrain the selection of relevant features, machine learning algorithms typically rely on  known structures in the data (for instance, networks, pathways, the linear structure of DNA). This type of prior knowledge can be taken into account with regularisation techniques (see, e.g.,~\cite{john@regularization,regularization}) or priors in Bayesian inference. Another efficient way to address the lack of data is to transfer knowledge across related tasks (see~\cite{PanY09TKDE} for a survey).

\subsection{Expert Knowledge Elicitation}

Human judgement, and in particular expert knowledge, is often of crucial importance in decision making processes. Expert knowledge elicitation techniques have been widely studied in a wide range of application settings, from preference model elicitation \cite{azari2013preference,busa2014preference} to 
medical science and shipping industry~\cite{flores2011incorporating,zhang2016expert}, as well as interactive learning for student modelling~\cite{conati2002using,masegosa2013interactive}. 

The methodological choices needed in defining the expert feedback differ from application to application and depend on multiple conditions, such as (i) the available expert knowledge (e.g., What can one ask from the expert? How much should the expert be trusted?); (ii) the type of information to be obtained (e.g., learning a preference, estimating a quantity, answering a question, identifying risky options); (iii) knowledge extraction constraints (e.g., time/cost needed to get the answer, how many interactions/how much feedback can the expert provide).

Eliciting coefficients for linear regression methods has been shown to be efficient in previous related studies. An important line of work \cite{garthwaite1988quantifying, kadane1980interactive} studies methods of quantifying subjective opinion about the coefficients of linear regression models. The prior knowledge is elicited through tasks that use hypothetical data and the assessment of credible intervals. These elicitation methods are shown to obtain prior distributions that represent well the expert's opinion, but the use of expert knowledge is not explored further. In our approach, we elicit expert's opinion more directly and this also allows to focus on more specific cost functions (such as reducing the prediction error for a specific target).

In Bayesian inference the prior distributions of the parameters are a natural way of expressing prior knowledge. In the studies on prior elicitation \cite{OHagan06}, some also on regression \cite{OLeary09}, the focus has often been on how to elicit reliable prior knowledge, after which the Bayesian inference machinery takes care of the rest. We ask the complementary question, of how efficiently can the knowledge elicitation be done, first given the simplifying assumption that the expert feedback is reliable. The order of inference is also reversed: we initialize from data and then improve with knowledge elicitation, whereas in prior elicitation the normal order would be the opposite. In the next stages of development, it will be important to combine both approaches, and then the Bayesian formulations will be natural.

\subsection{Contributions and Outline}

In this work we propose to solve the ``extremely small $n$, large $p$'' problem for regression under the assumption that accurate expert knowledge is available but under a budget. As far as we know, this is the first study handling knowledge elicitation from this angle, aiming at sample sizes of $n\to 1$. This setup is important in particular for personalized medicine but not restricted to it.

The remainder of the paper is organised  as follows: In Section~\ref{s:prelim} we provide a detailed description of the expert knowledge elicitation scenario in which we address the lack of data problem in a regression task. We also state our assumptions about the use of the expert feedback. We then propose an effective algorithm for selecting on which features to ask expert feedback to reduce the loss the most, in Section~\ref{s:algo}. We analyse its optimality. In Section~\ref{s:experiments} we present simulations of the behaviour of our expert knowledge elicitation strategy in a simple synthetic setting. Then, by simulating a knowledgeable user, we show the potential prediction improvement when using a genomics dataset, relevant for the personalized medicine settings. To demonstrate the more general use of the algorithm that we propose, we also study its behaviour for model extensions including: a restriction of the features on which the expert can provide feedback (Section~\ref{s:extension_subset}) and by including noisy expert feedback (Section~\ref{s:extension_noisy}). Finally, we conclude and provide future work directions in Section~\ref{s:conclusions}.

\section{Preliminaries}\label{s:prelim}

In this section we present the problem setup and introduce the first formulation for expert knowledge elicitation in a prediction task. In particular, we explicate our assumptions about the type of feedback that the expert can give. The framework was chosen to be as simple as possible while still capturing the essential elements of large $p$, small $n$ data. For concreteness, we will describe the problem with terminology of treatment effectiveness prediction, but the setup is naturally more generally applicable.

\subsection{Problem Description}

The goal is to improve prediction of the effect of a treatment on a target patient, by including feedback provided by an expert. Assume a small set of observation data, which can be used to learn an initial predictor. The set consists of $n$ observed treatment responses $y_1,\dots, y_n$, stored in the vector $\Y \in \mathbb R^{1 \times n}$, coming from $n$ different patients $i$ ($i=1,\ldots,n$) who had previously received the same treatment. Denote the matrix of genomic features with $\X$, where the size of $\X$ is $n \times p$, and on each row $i=1,\dots,n$ we have the $p$ genomic features corresponding to patient $i$, denoted by $x_i$. We focus on setups with $p\gg n$. For the new ``target'' patient the same genomic measurements are available, denoted $x^* =  [x^*(1)~\dots~ x^*(p)]$, and the goal is to predict as accurately as possible the treatment response, denoted  $y^*$.

\subsection{Data Assumptions}

\paragraph{Linear Regression.} We assume there is a linear relation between the genomic features of the patient and the expected result of the treatment. More precisely, we choose a linear regression setting, where for each patient $i$,
\begin{align}
y_i = x_i \theta ^\top  + \eta, 
\end{align}
where $\theta \in \mathbb R^p$ is an unknown parameter underlying the linear function and $\eta$ is i.i.d white noise, quantifying the inherent noise in the measurements of the drug effects for each patient. The coordinate $\theta(i)$ 
of the parameter vector encodes the weight or relevance that feature $i$ has in computing the treatment response of a patient. 


\paragraph{Sparsity.} We assume that the weight vectors $\theta$ are $s$-sparse $(s<<p)$, or in other words that many of the features have zero weight or relevance in the drug response prediction. Note that sparsity is not necessary for expert knowledge elicitation, but is a widely used regularity assumption which enables handling even larger $p$. Sparsity assumption matches many problems having a small number of responsible mechanisms (for instance, mutations in the genetic cases).



\paragraph{``Small $\textbf{n}$'' and ``$\mathbf{n=1}$'' scenarios.} An important data assumption is whether the true weight of a feature is the same for all patients (that is, do all observations $(x_1,y_1),\dots,(x_n,y_n)$ come from the same distribution), or whether the target patient may be from a different distribution. We call the former the \emph{``small $n$ scenario''} and the latter the \emph{``$n=1$ scenario''}.

A particularly useful variant of the $n=1$ scenario is the \emph{``multiple $n=1$ tasks scenario''}, where each patient has his/her own (unknown) weight vector; for $i \neq j$, $\theta_i \neq \theta_j$, and
\begin{align}\label{eq:observation_model.multitask}
y_i = x_i \theta_i^\top  + \eta ~~~\text{and}~~~ 
y_j = x_j \theta_j^\top  + \eta' \; ,
\end{align}
where $\eta$ and $\eta'$ come from $\N(0,1)$. If we are willing to assume that all considered patients share the same sparse support of active genomic features (that is, features affecting the drug response) and that their corresponding weights are similar 
from patient to patient, then we can estimate an initial estimate for the weights $\theta^*$ of the target patient from the other patients' data. 

We next explicate our assumptions about the information that can be extracted from an expert in the linear regression scenario.

\subsection{Expert Feedback}\label{s:expert_knowledge}

For clinical and behavioural variables, an expert may know how much they explain of the risk\footnote{For instance, even wikipedia gives an estimate that obesity appears to be the cause of 20\% of heart attacks.}. For linear regression, the fraction of variance explained is the square of the correlation coefficient (under simplifying assumptions). When the expert is more uncertain, he/she can give feedback on the importance of variables. For instance, some cancer genes and pathways are known, and given the patient's treatment response history, it is possible to make educated guesses of which hypotheses of disease mechanisms remain as potential hypotheses. In these cases our formulation is an approximation for bringing in expert's patient-specific prior knowledge.

\paragraph{Expert knowledge.} We assume that the expert is able to report the correct value of $\theta^*(i)$ when asked, but that answering requires a cost and hence we cannot simply ask correct values for the full $\theta^*$. This kind of feedback is very informative and, as far as we know, has not been used in estimating parameters before. This assumption is very simplifying in the personalized medicine case, and requires that the expert either has important additional knowledge of the particular patient, or is able to use his/her expertise to infer the correct value from the shown data $x^*$ and the initial weight vector estimated based on the other patients. We later relax this assumption and provide an empirical study of the sensitivity to expert errors (Section \ref{s:extension_noisy}) and to expert knowledge restricted to a subset of features (Section \ref{s:extension_subset}).

\paragraph{Feedback use.} In the simple formulation assuming accurate experts, the best way of taking into account the expert feedback on a feature $i$ is to directly replace the feature weight of the estimated target parameter $\htheta_\text{init}$  (obtained from the data of other patients) with the target-specific weight provided by the expert. Basically, if the expert gives feedback on $\theta^*(i)$, then we update the initial ``small $n$'' estimate (denoted $\htheta_\text{init}$) by replacing its $i$-th coordinate with the feedback provided by the expert.

\paragraph{Feedback cost.} We assume this type of expert knowledge to be very expensive and we hence place a strict restriction on the number of features on which the expert can provide feedback. Denote by $m$ the number of features for which the expert can give target-specific information (that is, the corresponding weight to be considered in estimating the drug response of the target patient). We refer to $m$ as the \textit{feedback budget} and we restrict it to a value much lower than the dimensionality of the data: $m << p$. Therefore, assuming the user does have the answer, the research problem we address here is that of identifying the $m$ most informative features on which to elicit expert knowledge.

As our goal is to predict the drug response for patient $x^*$, it follows that the most informative feedback the expert can provide is what leads to minimizing the prediction error. We formalize this in the performance measure defined below.

\subsection{Performance Measure}

Let $\htheta_{\A}$ denote the estimate of $\theta^*$ produced by an algorithm $\A$. We define the loss of $\A$ as the expected quadratic loss for the target patient $x^*$: 
\begin{align}\label{eq:loss_target}
L_{\A}= \E [(x^{*} \htheta_{\A}^\top - y^*)^2] = \E [(x^{*}\htheta_{\A}^\top - x^{*}\theta^{*\top})^2],
\end{align}
where the expectation is taken over all sources of noise, coming from the  noisy drug effect observations and noisy selection of the features by the expert. Given the $m$ interactions with the expert user, our goal is to get feedback about the most \textbf{informative or relevant} features, such that we minimize the target loss. The optimal algorithm $\mathcal A^*$ is thus defined as:
%
$L_{\A^*} = \min_{\A | \htheta_\text{init}} L_{\A}.$
%
It is important to highlight that the overall performance of an algorithm strictly depends on the estimate obtained from the training data on other patients $\htheta_\text{init}$.  In this work the focus is on finding expert knowledge elicitation strategies that enable maximally improving the target prediction, from the initially very imperfect estimate obtained from scarce data, possibly coming from different distributions.

\section{Algorithm}\label{s:algo}

We propose to learn the regression parameters in two stages. First, an initial estimate $\htheta_{\text{init}}$ is learnt on the ``small $n$ large $p$'' training data with appropriate regularization, efficiently capturing the information in that data set. The estimate is then improved by knowledge elicitation from an expert. The estimation error obtained with the initial estimate $\htheta_{\text{init}}$ is the baseline for comparing the potential improvement obtained by eliciting expert knowledge.

\subsection{Description}

Given the assumed type of feedback (described in Section~\ref{s:expert_knowledge}) and the budget constraint, the goal is to rapidly identify the features most useful in reducing the loss. After a closer look to the loss definition, we easily get the following intuitions on the behaviour of the algorithm:
\begin{itemize}
\item Trivially, the algorithm should not ask feedback more than once on the same feature $i$, since the expert feedback is accurate and we get directly the right value for $\theta^*(i)$.
\item By decomposing the loss as a sum of $p$ squared point-wise products, we can tell that it is most useful to reduce to 0 (through  the expert feedback) the $m$ largest feature products $x^*(i) \cdot \left(\theta^* - \htheta_{\text{init}}\right) (i)$, which would correspond to getting feedback on the worst estimated features in $\htheta_{\text{init}}$.
\end{itemize}
Of course, the true vector $\theta^*$ is not available and the best proxy we can have for it is given by  $\htheta_{\text{init}}.$ Therefore, we propose to use the available information on $x^*$ and the estimate of the weight vector, and we ask for expert feedback on the $m$ largest features, as given by the point-wise product of $x^*$ and the $\htheta_{\text{init}}$.

For each feature $j=1,\dots,m$, we replace the coordinate $\htheta_{\text{init}}(j)$ by the corresponding expert feedback. The rest of the $p-m$ features remain the same.  The algorithm outputs the vector thus obtained. We denote the result by $\hat \theta_{\text{final}}$ and the loss will be given by
%
$L= \E [(x^{*}\htheta_{\text{final}}^\top - x^{*}\theta^{*\top})^2]$. 
%
The pseudo code of the algorithm is presented in Algorithm~\ref{alg:expert-elicitation-regression}. 
\begin{algorithm}
\begin{algorithmic}
\STATE \textbf{Input:} $ \{\X, \Y \}$ -- $n$ training data samples in $\mathbb R^p$ \\
~~~~~~~~~~~$x^*$ -- feature vector representing the target\\
~~~~~~~~~~~$m$ -- expert elicitation feedback budget\\
\STATE \textbf{Initialization phase}\\
Output: $\htheta_{\text{init}}$ -- estimated weight vector from data
\STATE \textbf{Expert elicitation phase}
\STATE Rank point-wise products of features
\FOR {$t = 1, \dots, m$}  
	\STATE Elicit feedback on each of the $t$-th largest product feature 
    \STATE Get the corresponding $\theta^*$ coordinate: 
    $r = \theta^*(i)$
	\STATE Replace feature in the initial estimate: 
	$\htheta(i)= r$
\ENDFOR
\STATE Output $\hat \theta_{\text{final}}$
\STATE \textbf{Goal:} minimize~ $L = ( x^{*\top} (\hat \theta_{\text{final}} - \theta^*))^2$.
\end{algorithmic}
\caption{}
\label{alg:expert-elicitation-regression}
\end{algorithm}
%

\subsection{Analysis}\label{s:proof}

We will next give conditions under which the elicitation sequence of Algorithm~\ref{alg:expert-elicitation-regression} is optimal. Denote $\Delta_i = x^*(i) \theta(i) - x^*(i) \theta^*(i)$, and index the feature with the largest product of feature value and regression weight with $c$. The theorem below says essentially that assuming changes in regression parameters from learning set to the target patient are on average monotonically increasing as a function of the parameter size, choosing the largest product $c$ decreases cost function more than any other choice on the average. An additional requirement is that if there are correlations between parameters, they cannot be stronger for others compared to $c$. The averages are over variation in the learning data set.

\begin{theorem}
Denote $c=\argmax_k x^*(k) \theta(k)$ and assume $E[\Delta_c^2] \geq E[\Delta_i^2]$ and $E[\Delta_c\Delta_k] \geq E[\Delta_i\Delta_k]$ for all $i \neq k \neq c$. Then it holds that
\begin{align*}
    L_{\theta^{c*}} \leq L_{\theta^{i*}},
\end{align*}
where $L_{\theta} = E[(x^{*}\theta^\top-y^*)^2]$ and $\theta^{i*}=[\theta(1),\ldots,\theta(i-1),\theta^*(i),\theta(i+1),\ldots,\theta(p)]^\top$.
\end{theorem}

%
%

\begin{proof}
Since $L_\theta = E[(\sum_k \Delta_k)^2]$,
\begin{align*}
    L_{\theta^{i*}} &= E[(\sum_{k\neq i} \Delta_k)^2] \\
                    &= L_\theta - 2 E[\Delta_i \sum_k \Delta_k] + E[\Delta_i^2].
\end{align*}
Hence,
\begin{align*}
    L_{\theta^{c*}} - L_{\theta^{i*}} = E[\Delta_c^2] - E[\Delta_i^2] - 2E[(\Delta_c -\Delta_i) \sum_k \Delta_k].
\end{align*}
The expression within the last expectation is
\begin{align*}
    (\Delta_c - \Delta_i) \sum_k \Delta_k &= (\Delta_c - \Delta_i)(\Delta_c + \Delta_i + \sum_{k \neq c,i} \Delta_k)\\
    &= \Delta_c^2 - \Delta_i^2 + (\Delta_c - \Delta_i) \sum_{k \neq c,i} \Delta_k,
\end{align*}
and therefore
\begin{align*}
    L_{\theta^{c*}} - L_{\theta^{i*}} &= 
    E[\Delta_i^2] - E[\Delta_c^2] 
    - 2 \sum_{k \neq c,i} (E[\Delta_c \Delta_k] - E[\Delta_i \Delta_k]) 
    \leq 0
\end{align*}
by the two assumptions.
\end{proof}

\paragraph{Optimality in a simple setting.}

We next illustrate the theorem in the simple setting where the target patient comes from the same distribution as the other patients. Then under reasonable regularity assumptions $E[\theta] = \theta^*$, and
\begin{align*}
    E[\Delta_i^2] &=  
    x^2(i) E[\theta^2(i)] -2 x^2(i) E[\theta(i)] \theta^*(i) + x^2(i) \theta^{*2}(i) \\
    & = x^2(i) (E[\theta^2(i)] - \theta^{*2}(i)) \\
    & = Var[x(i)\theta(i)].
\end{align*}
Skipping analogous details,
\begin{align*}
    E[\Delta_i \Delta_k] = x(i)x(k)(E[\theta(i) \theta(k)] - \theta^*(i)\theta^*(k)).
\end{align*}
Hence the assumptions translate in this case to the intuitive requirements that variance in the largest features is the largest, and features are either not cross-correlated, or if they are, cross-correlations with the largest feature are the strongest (on average).

\section{Experiments}\label{s:experiments}

We illustrate the performance of Algorithm~\ref{alg:expert-elicitation-regression}, to which we refer in the sequel by \textsc{Largest Product Feature}, in two experimental setups. We start with a simple synthetic setting (described in Section~\ref{s:synthetic}), then we use a genomics dataset for a more elaborate simulation (as described in Section~\ref{s:dataset}). In both settings we compare the loss of Algorithm~\ref{alg:expert-elicitation-regression} to that of the following strategies:
\begin{itemize}
    \item \textsc{No interaction}: the baseline algorithm whose performance is given by the prediction error of $\htheta_{\text{init}}$
    \item \textsc{Random}: works by selecting at random (without repeat) $m$ features of which to ask expert feedback
    \item \textsc{Largest Target Feature}: asks feedback on the $m$ largest coordinates of the target vector $x^*$. 
\end{itemize}
While \textsc{No interaction} and \textsc{Random} are typical baselines, \textsc{Largest Target Feature} is a naive approach of minimising the target loss, based solely on the absolute feature values of the target vector.

\subsection{Synthetic Data}\label{s:synthetic}

\paragraph{Setting.}
We randomly generated the training set $X$ having $k=1000$ rows and 150 features from a normal distribution with mean 0 and variance 1. We also randomly sampled a sparse weight vector $\theta$, such that 5 of its features are non-zero and come from a normal distribution with mean 0 and variance 1, while the remaining 145 features are 0. The output variable $Y$ is then computed using $n$ noisy observations of the dot product between $\theta$ and randomly selected vectors $x \in \X$. We use the glmnet package~\cite{friedman2010regularization} for estimating $\htheta_{init}$, and we vary the number of training samples from  5 to 30, while the number of expert feedbacks that we assume we can obtain grows from 0 to 10. We randomly choose a target patient and compute the corresponding loss for each feedback value. We plot the average target loss over 100 randomly selected target patients. 

\paragraph{Results.} Figure~\ref{fig:synthetic_single_theta} shows the prediction performance in terms of a loss for a target sample in the ``small $n$'' scenario, with four different strategies. As expected, given that the expert is assumed to be able to give exact feedback, all strategies that use expert feedback show improvement in performance as the number of expert feedback grows and have better performance than the baseline (\textsc{No Interaction}) from the very first expert feedback. It can also be noticed that the initial estimate of the weight vector resulting using \textsc{Largest Product Feature}, after the first feedback is already better then the other strategies. Then, with increasing amount of feedbacks received, our algorithm shows increased improvement in performance as compared to other strategies. 

In the ``$n=1$'' scenario, with different regression parameter vectors for each patient, the improvement is slightly slower (because of the introduced bias in the initial estimate $\htheta_{\text{init}}$) but still clearly better than with the other strategies (Figure~\ref{fig:synthetic_multi_theta}). It is important to mention that the $\theta$s ($\theta_1, \dots, \theta_k$, one corresponding to each patient in $X$) were randomly generated, but to control the introduced bias we keep the same sparsity assumption for all weight vectors (the same 5 nonzero features) and we restrict the $L_2$ norm of the difference between any pair of weight vectors $\theta_i \neq \theta_j$ to be smaller than 0.5.

For both scenarios, here we only report the loss when the number of training samples used for computing $\theta_{\text{init}}$ is 10, but in Appendix~\ref{s:appendix} we report more complete results.

\begin{figure}[h]
\centering
    \subfigure[Small $n$ scenario]{\label{fig:synthetic_single_theta}\includegraphics[width=0.45\textwidth,keepaspectratio]{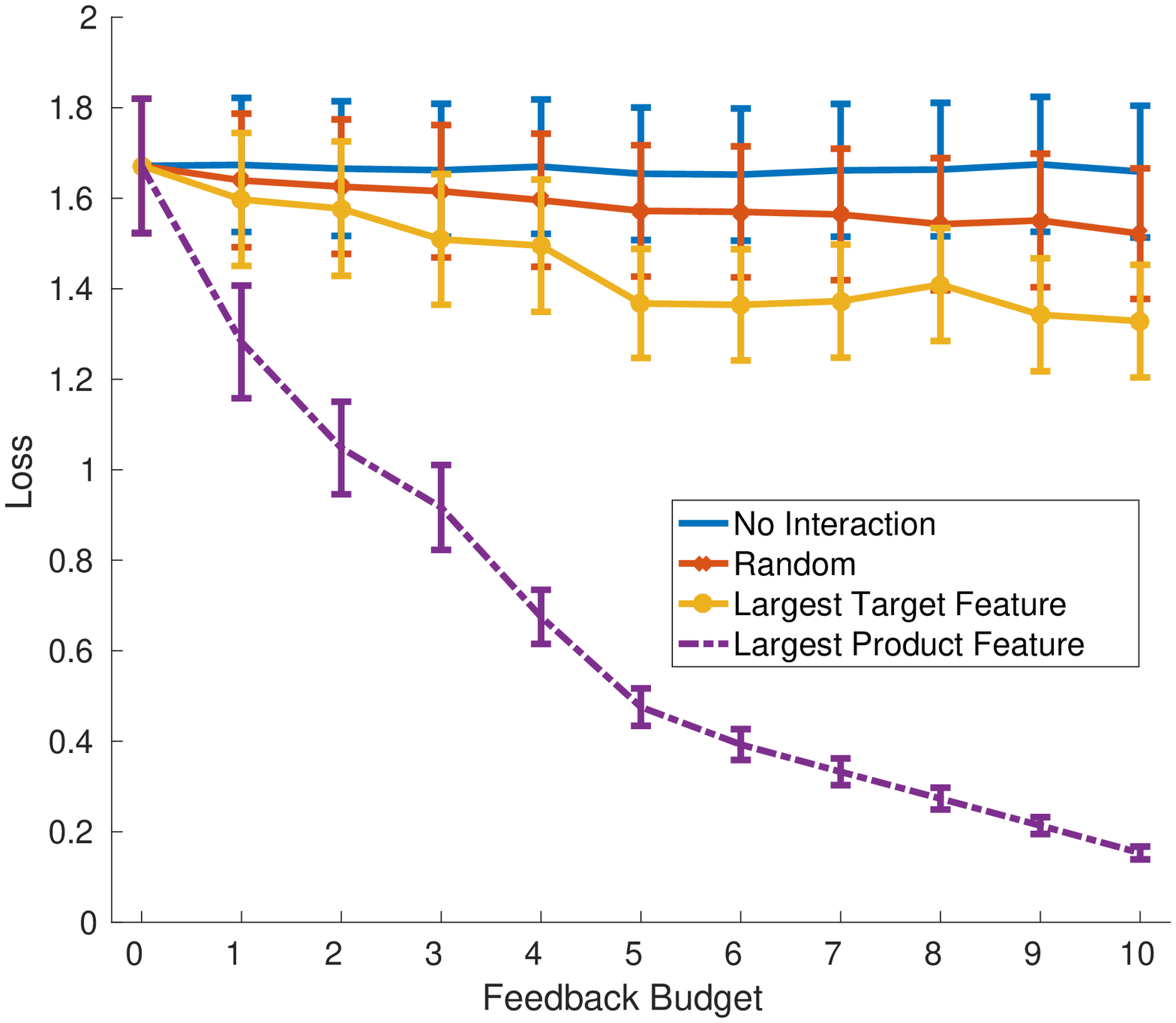}
		}
    \subfigure[$n=1$ scenario]{\label{fig:synthetic_multi_theta}
        \includegraphics[width=0.45\textwidth,keepaspectratio]{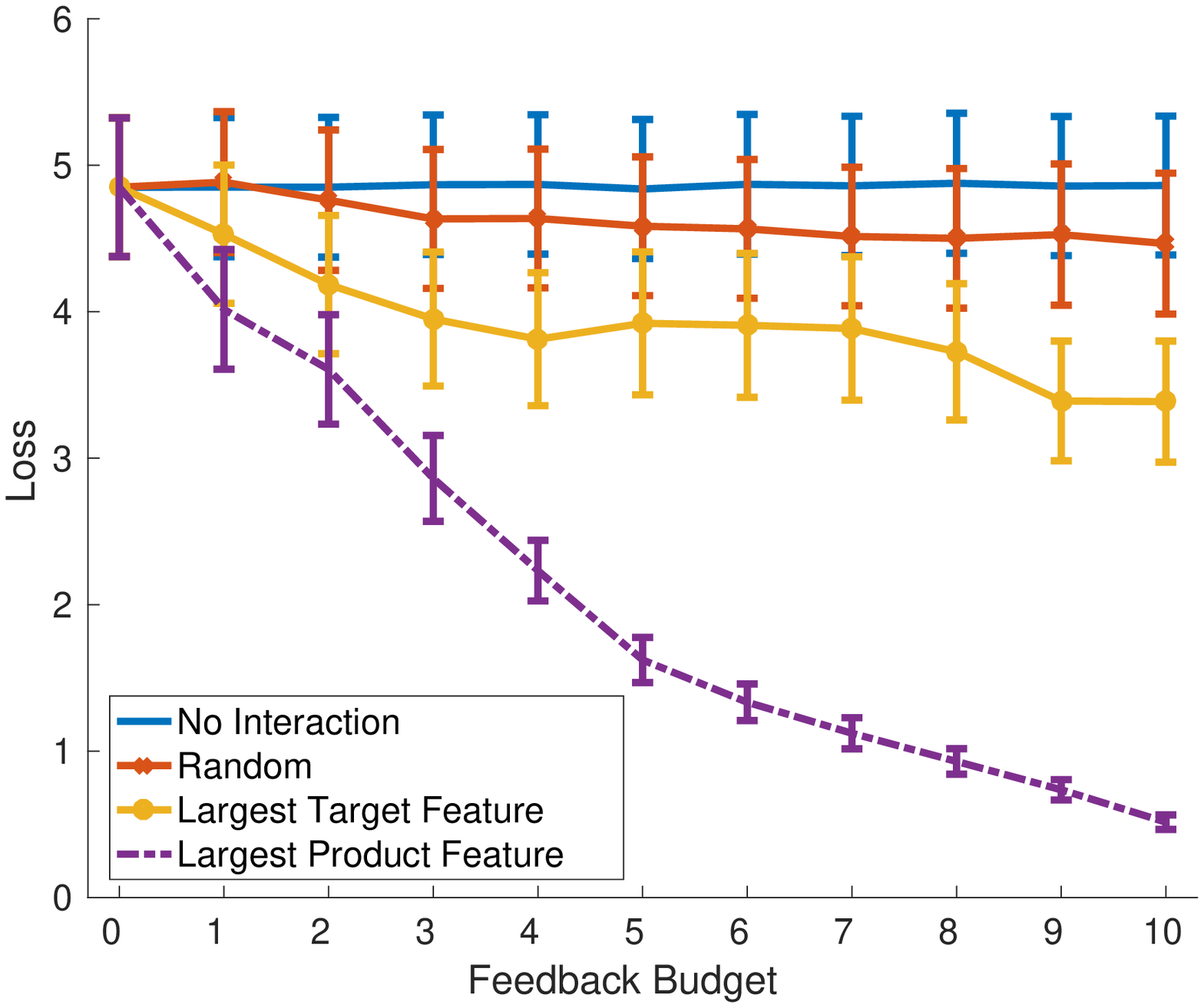}
    }
\caption{Loss for one target patient in the ``small $n$'' scenario where each patient has a same parameter vector \ref{fig:synthetic_single_theta}, respectively loss for one target patient in the scenario where each patient has a different parameter vector \ref{fig:synthetic_multi_theta}. X-axis denotes the feedback budget and Y-axis the loss in predictions, averaged over separate predictions for each target patient. The curves show the mean loss and the bars the standard error of the mean.}
\end{figure}

\subsection{GDSC Dataset}\label{s:dataset} 

We demonstrate the usefulness of our approach by also testing it on the genomics data in the GDSC dataset. Here we obtained results similar to the results seen from the synthetic data. In the following we briefly describe the contents of the dataset, then we explain how we simulated the ground truth. The plotted results (Figure~\ref{f:realdata1}) follow the same trends as in the synthetic data, again showing an improvement of all the strategies that use interaction, for all values of $n$.

\paragraph{Genomics of Drug Sensitivity in Cancer (GDSC) data.}

We used the data from the Genomics of Drug Sensitivity in Cancer project by Wellcome Trust Sanger Institute (version release 5.0, June 2014, \url{http://www.cancerrxgene.org}) \cite{yang2013genomics,garnett2012systematic} consisting of $124$ drugs and a panel of $124$ human cancer cell lines for which complete drug response measurements are available. Drug responses are summarised by log-transformed IC50 values (the drug concentration yielding 50\% response, given as natural log of $\mu M$) from the dose response data measured at 9 different concentrations. The cancer cell lines, representative of human cancer cell models, are characterised by expression values quantifying the transcript levels of thousands of genes. For this study, we chose a subset of biologically relevant genes, whose mutational status has been shown to correlate with the drug responses~\cite{WelTSanIns}. For our study, we transformed the transcript counts of the genes to the $\log2$ scale. 

\paragraph{Learning a Pseudo-Ground Truth.} When simulating expert feedback, on this data we used estimates computed from the full data as the correct answers the expert gives, when queried based on estimates from small $n$.
To learn this ``pseudo-ground truth,'' we employed sparse linear regression using the {\it glmnet} package, which has been frequently used to identify genomic features of drug responses~\cite{barretina2012cancer,garnett2012systematic}. The sparse linear regression formulation has two parameters that are to be optimized: $\alpha$ (elastic net mixing parameter) and $\lambda$ (the penalty parameter). We fixed $\alpha=1$ assuming only few of the genomic features to be important for predictions. We inferred $\lambda$ as follows:
For each drug, we held out one cell line and performed a 10-fold cross validation procedure on the training data with 100 different $\lambda$ values. The training data comprised of the gene expression values of all cell lines and their responses on this drug, except for the held-out cell line. 
We chose $\lambda_{min}$ that gave minimum error averaged over the 10 cross-validated folds. The estimates of $\theta$ were obtained by setting $\lambda=\lambda_{min}$. We repeated this procedure for all drugs, thus the learnt $\theta$s are used as a pseudo-ground truth for our experiments.     

\begin{figure*}
\hspace{-0.8cm}
\includegraphics[trim={0 6cm 0 0},width=1\textwidth,keepaspectratio]{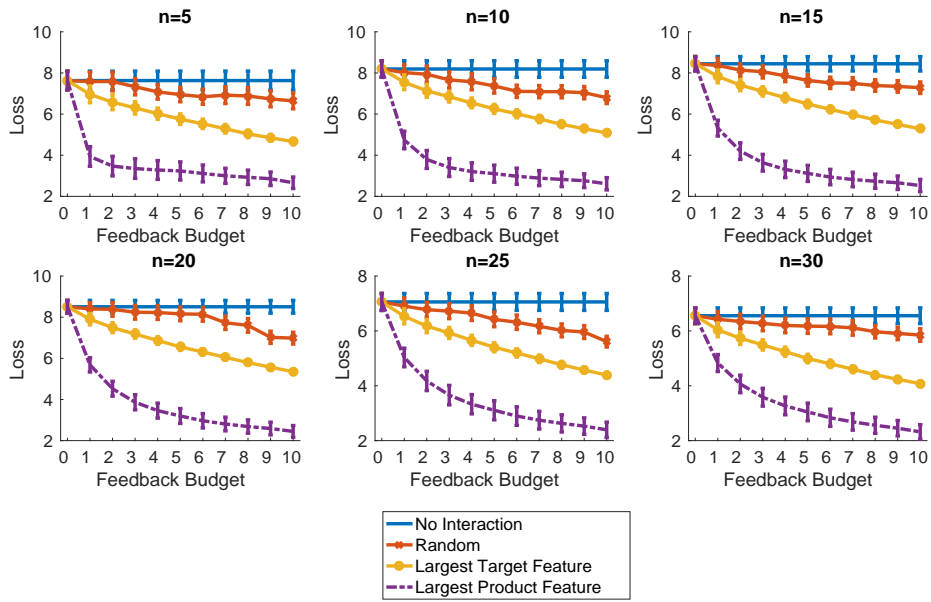}
\caption{GDSC dataset: Loss for one target patient in the scenario where each patient has a different parameter vector $\theta$. In each subplot, X-axis denotes the feedback budget and Y-axis the loss in predictions, averaged over separate predictions for each target patient. The curves show the mean loss and the bars the standard error of the mean. $n$ denotes the number of small sample sizes used for computing $\hat\theta_{\text{init}}$. \label{f:realdata1}}
\end{figure*} 

\paragraph{Results.} For the plots on the GDSC dataset, for each amount of feedback and for each algorithm, the curve represents the average loss over 100 random iterations. The bars are showing the standard error of the mean. In each iteration, we randomly pick a set of 10 target patients and 10 drugs and we compute their corresponding losses. For each of the target patients we predicted the drug responses eliciting the expert knowledge of the user following the strategies presented in the beginning of Section~\ref{s:experiments}. We also show the effect of varying the number of initial training samples, from 5 to 30.

While the trends and the ordering of the performance of the algorithms do not change, we can notice that, as expected, the overall loss of the algorithms diminishes as $n$ grows. In addition, when using the GDSC dataset the loss does not decrease to zero as fast as for the synthetic data. 

\section{Model Extensions}\label{s:extensions}

To get more evidence about the behaviour of the algorithm, we proceed by relaxing the model in two central aspects. These relaxations also make the model applicable in a considerably wider set of practical scenarios. For the numerical simulations, unless specified otherwise, the setting remains the same as described in Section~\ref{s:synthetic}, for the ``n=1'' scenario.

\subsection{Feedback on a subset of features}\label{s:extension_subset} 

We analysed the prediction performance of our approach when the expert can provide feedback only on a subset of features. To simulate that, we associate to each feature a value of either 1, meaning that the expert has knowledge on that feature (and can provide feedback), or 0, for the features on which the expert has no prior knowledge. The algorithm \textsc{Largest Product (Subset) Feature} selects the features on which to ask feedback as before, but if the selected feature is unknown by the expert (that is, its associated value is 0), then feedback is asked on the next ``largest product feature'' with an associated value 1. Basically, instead of receiving feedback on the $m$ largest features, we now receive expert feedback on the $m$ ``largest product features'' with associated value 1. 

The analysis in Section~\ref{s:proof} also holds for this more general setting, under the same assumptions. In fact, we again proceed by asking expert feedback on the features that allow to decrease the cost function more than any other choice (on average), since the features on which the expert has no knowledge would have no impact on the loss (if selected). Thus the algorithm preserves optimality, since by using the feedback budget on the $m$ ``largest product feature'' with an associated value 1, we obtain the largest expected loss reduction.

Figure~\ref{f:synthetic_subset_feedback} shows how the performance behaves when the percentage of features with associated value 1 is reduced from 90\% to 50\%.\footnote{Note that the plot for ``Feature Subset=100\%'' would be precisely the Figure~\ref{fig:synthetic_multi_theta} (with overlapping green and magenta curves).} 
Each subplot corresponds to a different proportion of features (selected randomly) on which the expert has prior knowledge and can provide feedback. For instance ``Feature Subset=90\%'' means that expert feedback is available on 90\% of the features in the curve \textsc{Largest Product (Subset) Feature}.

Based on the results, even in the more realistic cases in which the expert knowledge is restricted to a subset of the features, the results obtained by \textsc{Largest Product (Subset) Feature} are significantly better than the baseline \textsc{No Interaction}. Note that the other strategies remain unchanged, with no constraint on the features on which expert feedback can be received, and yet \textsc{Largest Product (Subset) Feature} outperforms them even with half of the total number of features. 

\begin{figure*}
\hspace{1.3cm}
\includegraphics[width=0.72\textwidth]{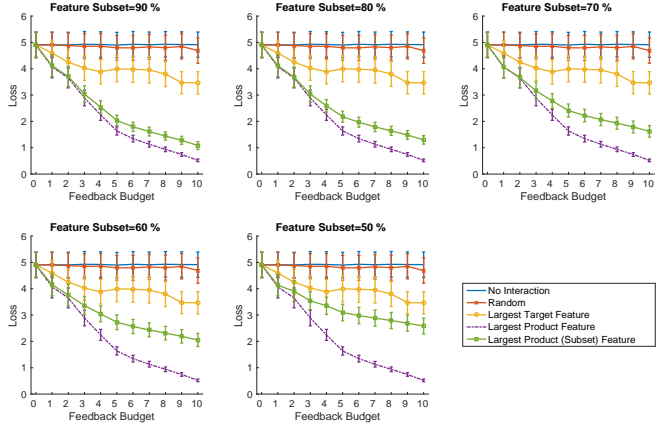}
\caption{The effect of expert having knowledge on only a subset of features. The curves show the mean loss in predictions, for one target patient in the scenario where we consider a different (but similar) $\theta$ corresponding to each patient. We plot the loss averaged over 100 randomly selected target patients, obtained with 10 training samples drawn using the synthetic data. X-axis denotes the feedback budget, and the bars show the standard error of the mean.}
\label{f:synthetic_subset_feedback}
\end{figure*}

\subsection{ Noisy Expert}\label{s:extension_noisy}  

We simulated an experiment for the ``n=1'' scenario where each feedback is affected by normally distributed noise, centered and with variance between 0.1 and 0.5, encoding the expert uncertainty (the range of the (true) feature values is [0, 1]). The noisy feedback is used for all strategies. In the personalized medicine application, it is plausible to assume that the expert feedback has smaller variance in the weight of genomic features commonly known to be relevant, but much higher variance for rarely encountered genes and their mutations. 

In Figure~\ref{f:synthetic_noisy_expert} we can see that \textsc{Largest Product Feature} is still better than the baseline \textsc{No Interaction} in the presence of noisy feedback. Though as the noise increases the difference between the prediction performance of our algorithm and the baseline decreases. This is expected, since as the expert provides noisier feedback on a feature $i$, $\hat\theta(i)_{\text{final}}$ might deviate even more from the true value of feature $i$ than the $\hat\theta(i)_{\text{init}}$.

Regarding the other strategies, we can see that when the expert feedback has a variance greater than 0.3, the baseline has a better performance than \textsc{Random}, which asks feedback on features chosen at random (without repeat). For \textsc{Largest Target Feature} the difference with the baseline is even more significant. We can notice a decrease in performance with every additional feedback, for a noise variance of more than 0.3. This happens because when computing the loss, for the parameter features whose value is changed (from $\hat\theta_{\text{init}}(i)$ to $\hat\theta_{\text{final}}(i)$), the introduced noise is then multiplied with one of the $m$ largest features of the target patient $x(i)$. For this strategy, the noisy feedbacks propagate the most and decrease the performance, leading to results worse than \textsc{Random}.

Although a natural extension, strategies require much larger feedback budgets to obtain effective expert knowledge elicitation when the feedback is very noisy. In fact, the effect of a noisy feedback naturally reduces if one asks feedback multiple times on the same feature. But this would also imply a change in the assumptions about the structure in the data or about the number of experts that can be consulted. The design of an optimal strategy becomes more intricate in this scenario since it involves choosing the right balance between (a) obtaining noisy information about more coefficients, or (b) focusing on a smaller number of coefficients for which feedback is asked multiple times. This trade-off might also vary depending on the target, since its features will then propagate the uncertainty of the coefficients into the loss.

\begin{figure*}
\centering
\includegraphics[width=0.7\textwidth]{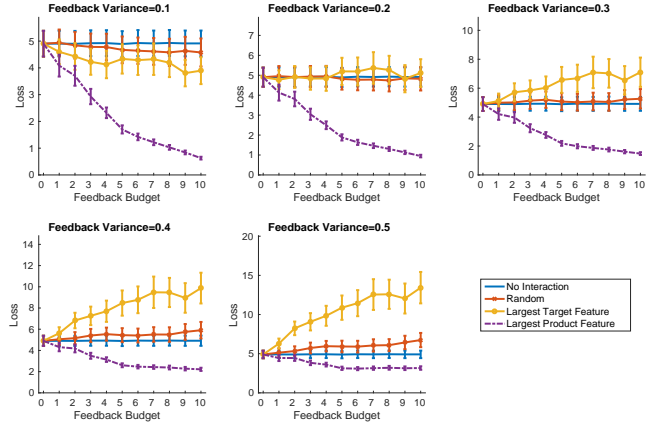}
\caption{The effect of feedback from a noisy expert on the loss for one target patient, when 10 samples from the synthetic data have been used to obtain $\hat\theta_{\text{init}}$. X-axis denotes the feedback budget and Y-axis the loss in predictions, averaged over separate predictions from 100 randomly selected target patients. The curves show the mean loss and bars the standard error of the mean. Each subplot corresponds to a different noise variance introduced in the feedback.}
\label{f:synthetic_noisy_expert}
\end{figure*}

\section{Conclusion}~\label{s:conclusions}

We have introduced a novel setup that brings together expert elicitation and the difficult ``small n, large p'' regression problem. Starting from noisy estimates based on extremely small sample sizes, we empirically demonstrated the prediction improvement that can be obtained by bringing in only a few expert feedbacks. More precisely, we considered a simplified problem formulation, where there is a strict budget constraint on exact expert feedback. The simplified problem setting is intended to be a starting point that opens up both new interesting theoretical questions and a line of applied work towards new solutions in the currently very timely problem of personalized medicine. Underlying the practically important goal of developing better predictions of treatment outcome for an individual patient, is the task of estimating predictors for the sample with $n=1$, which requires creative solutions. New approaches of querying and incorporating available expert knowledge are naturally expected to have much wider applicability.

For future work, a sensible formulation for the ``small n, large p'' regression problem is to find the optimal expert queries for reducing the interval uncertainty for regression coefficients, with strategies recently studied and applied  in reliability analysis problems \cite{Expert-elicitation-interval-uncertainty}. Another particularly appealing future formulation is adaptive expert feedback elicitation, where after each feedback, the estimate is updated and the next feature is sampled taking into account the current estimate. Similar expert interaction approaches were shown to be effective for user intent modelling in~\cite{Ruotsalo2013}. Lastly, for the elicitation method proposed here, we intend to run a full-blown user study and test the actual assessments from experts.

\bibliographystyle{plain} 


\newpage

\appendix

\section{Supplemental Empirical Results}~\label{s:appendix}
%
\vspace{-0.5cm}
\begin{figure}[H]
\centering
\includegraphics[trim={0 5.5cm 0 0},width=0.8\textwidth,keepaspectratio]{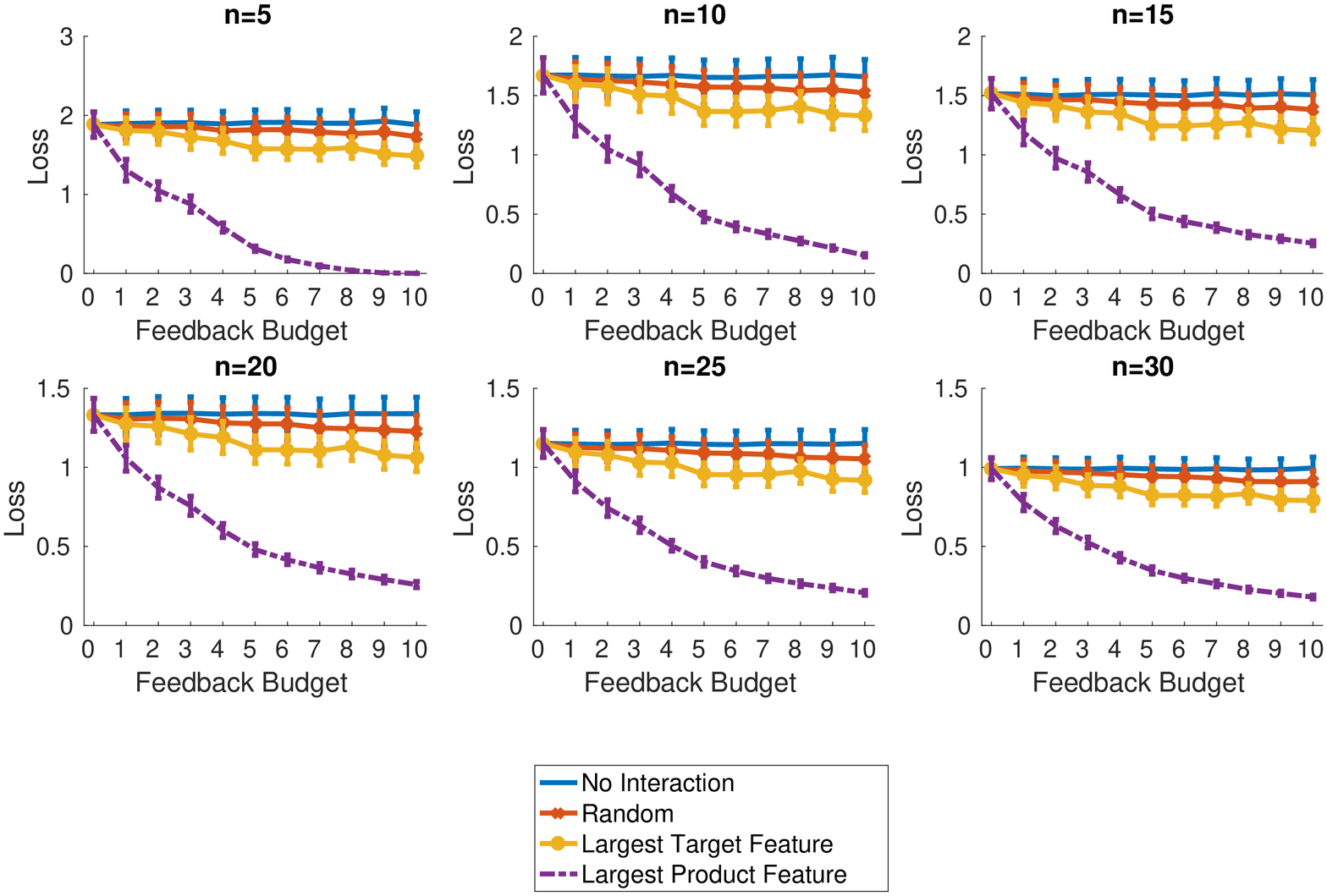}
\caption{Loss for one target patient in the ``small $n$'' scenario where the patients have the same parameter vector. X-axis denotes the feedback budget and Y-axis the loss in predictions, averaged over separate predictions for each target patient. The curves show the mean loss and the bars the standard error of the mean. Each subplot corresponds to the number of samples $n$ used in estimating $\htheta_{init}$~($n$ from 5 up to 30 samples).}
\label{f:synthetic3}
\end{figure}
\begin{figure}[H]
\centering
\includegraphics[trim={0 5.5cm 0 0},width=0.8\textwidth,keepaspectratio]{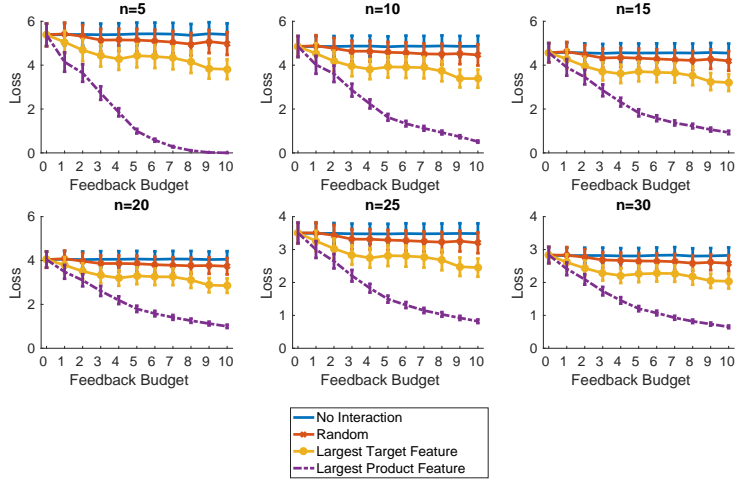}
\caption{Loss for one target patient in the scenario where for each $x_i$ there is a different weight vector $\theta_i$.  X-axis denotes the feedback budget and Y-axis the loss in predictions, averaged over separate predictions for each target patient. The curves show the mean loss and the bars the standard error of the mean. Each subplot corresponds to the number of samples $n$ used in estimating $\htheta_{init}$.}
\label{f:synthetic4}
\end{figure}
\vspace{-0.5cm}
\end{document}